%% file: main.tex
\documentclass[letterpaper, 10 pt, conference]{ieeeconf}  

\IEEEoverridecommandlockouts                              

\include{preamble}

\renewcommand{\baselinestretch}{0.963} 

\title{\LARGE \bf Learning Responsibility Allocations for Safe Human-Robot Interaction with Applications to Autonomous Driving
}

\author{Ryan K. Cosner, Yuxiao Chen, Karen Leung, and Marco Pavone
\thanks{ 
        Ryan K. Cosner is with the California Institute of Technology, {\tt\small rkcosner@caltech.edu}, 
        Yuxiao Chen is with NVIDIA, {\tt\small yuxiaoc@nvidia.com},
        Karen Leung is with the University of Washington and NVIDIA {\tt\small \{kymleung@uw.edu, kaleung@nvidia.com\}},
        Marco Pavone is with Stanford University and NVIDIA {\tt\small \{pavone@stanford.edu, mpavone@nvidia.com\}}
        Ryan's work was performed during an internship with NVIDIA.
}
}

\begin{document}
\maketitle
\thispagestyle{empty}
\pagestyle{empty}

\begin{abstract}

Drivers have a responsibility to exercise reasonable care to avoid collision with other road users. This assumed responsibility allows interacting agents to maintain safety without explicit coordination.
Thus to enable safe autonomous vehicle (AV) interactions, AVs must understand what their responsibilities are to maintain safety and how they affect the safety of nearby agents.
In this work we seek to understand how responsibility is shared in multi-agent settings where an autonomous agent is interacting with human counterparts. 
We introduce \textit{Responsibility-Aware Control Barrier Functions} (RA-CBFs) and present a method to learn responsibility allocations from data. By combining safety-critical control and learning-based techniques, RA-CBFs allow us to account for scene-dependent responsibility allocations and synthesize safe and efficient driving behaviors without making worst-case assumptions that typically result in overly-conservative behaviors. We test our framework using real-world driving data and demonstrate its efficacy as a tool for both safe control and forensic analysis of unsafe driving. 

\end{abstract}


\section{Introduction}
Drivers have a duty to exercise reasonable care when interacting with other road users. The assumption that other road users will exercise \textit{responsible} behaviors enables everyone to maintain safety without explicit coordination. 
Thus, for autonomous vehicles (AVs) to safely and seamlessly interact with other road users, they must utilize a safety framework that can interpretably codify basic safe driving behaviors (such as those laid out by the IEEE safe driving standards \cite{ieee2022P2846}), and exercise duty of care, i.e., drive responsibly.
However, the responsibility for ensuring safety is shared asymmetrically and the allocation is influenced by context-dependent social driving norms. For example, drivers are generally more responsible for avoiding collisions with vehicles in front of them than behind them (see Figure~\ref{fig:hero_figure}).

The goal of this work is to develop a rigorous yet flexible safety framework that is capable of codifying and synthesizing safe and responsible driving behaviors \cite{ieee2022P2846}.
Unfortunately, most existing models and techniques for ensuring AV safety tend to make strong assumptions about how other drivers behave which often results in defensive or erratic driving. For instance, many approaches make worst-case assumptions about the other agents' behaviors \cite{bansal_hamilton-jacobi_2017, nister2019safety, usevitch2022adversarial, shalev2017formal}; while these assumptions allow for strong theoretical guarantees, they are often impractical as they result in infeasible planning problems or induce overly conservative behaviors.
Therefore, a key challenge is developing safety models that are not only robust to varied driving behaviors, but that are also capable of accounting for context-dependent social norms that effect how drivers implicitly coordinate.


With that goal in mind, we propose Responsibility-Aware Control Barrier Functions (RA-CBFs), a \textit{responsibility-aware} safety paradigm that is conducive to efficient techniques for safe control synthesis and safety evaluation.
A primary use case of RA-CBFs is their application as a \textit{safety filter} within an AV stack, whereby the output of the main AV decision-making and planning algorithmic pipeline (powered, for example, by high-capacity human behavior prediction models based on deep learning) is monitored  for its compliance with basic safe and responsible driving behaviors (as encoded by the RA-CBF) and possibly revised in case of anomalies, in a manner similar to \cite{nister2019safety} and \cite{shalev2017formal}. As a proof of concept, we demonstrate this capability in the context of real-time closed-loop control and post-facto forensic analysis.

\noindent  \textbf{Related work.} We present relevant work on responsibility-aware safety methods for multi-agent interactions with an emphasis on AV applications. Specifically, we consider related work on socially-aware planning, safe multi-agent control, and safety constraint learning.

Many recent works focus on modeling and estimating drivers' social preferences to synthesize socially-aware driving behaviors. For instance, \cite{schwarting2019social} estimates the Social Value Orientation (SVO) of other drivers and formulates a SVO-based game-theoretic planner. \cite{toghi2021cooperative,sun2018courteous} craft a (potentially learned) planning reward function that incentivizes an AV to be more cooperative, sympathetic, and/or courteous. While these approaches demonstrate that accounting for social preferences can lead to more human-like AV behaviors, they do not provide any assurances or quantification of AV safety.

\begin{figure}
    \centering
    \includegraphics[width=\linewidth]{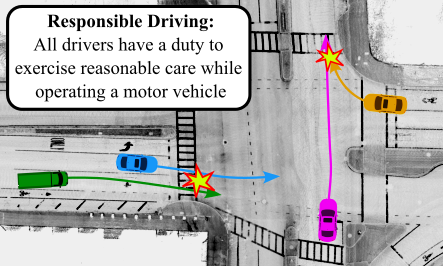}
    \caption{In human driving, vehicles can be expected to demonstrate a reasonable duty of care. For example, a trailing vehicle (green) takes responsibility for not colliding with the car in front of it (blue) and a merging vehicle (orange) follows formal and informal rules to avoid a collision with the vehicles in the lane (pink). How can we ensure that autonomous vehicles act according to such informal driving etiquette?}
    \vspace{-3em}
    \label{fig:hero_figure}
\end{figure}

In the safety-critical control literature, under decentralized multi-agent settings (which is more relevant to AV applications), many certifiably safe multi-agent control methods exist, including Control Barrier Functions (CBFs) \cite{usevitch2022adversarial,borrmann2015control, glotfelter_hybrid_2019, mustafa2022adversary}, Reciprocal Velocity Obstacles \cite{vandenberg2008reciprocal}, Hamilton-Jacobi Reachability \cite{chen2016multi,wang2020infusing}, Responsibility Sensitive Safety \cite{shalev2017formal}, and Safety Force Field \cite{nister2019safety}. However, these methods rely on worst-case or static assumptions on how other agents behave, which, unfortunately, do not hold in practice. This limitation has spurred recent works \cite{lyu2022responsibility, guo2021vr, chen2020guaranteed} that investigate collision-avoidance responsibility and how to decide which agent should take more effort to avoid collisions. However, these methods consider either centralized control or centrally defined social preferences, which does not apply to autonomous driving since the social preferences of other agents are not known precisely and cannot be assigned.

Learning responsibility allocation influenced by social norms necessitates the need to combine safety-critical control with data-driven methods. Recent safety-critical learning methods \cite{dawson2022safe,robey2020learning,qin2021learning,lyu2022adaptive} infer safety constraints from data. To the best of our knowledge, the consideration of responsibility has yet to be investigated within a data-driven safe multi-agent control setting. In this work, we take on a data-driven approach to learn how responsibility is allocated among multiple agents and synthesize corresponding safe \textit{and responsible} controls.
We elect to use CBFs as the basis of our method due to its interpretable and rigorous control-theoretic formulation and computational tractability.

\noindent  \textbf{Contributions}:
Our contributions are three-fold: (i) We present a novel concept of Responsibility-Aware Control Barrier Functions (RA-CBFs) which extends the standard CBF to account for asymmetric sharing of responsibility between multiple agents. (ii) We propose a data-driven constraint-learning algorithm to infer the responsibility allocations modeled in the RA-CBF formulation. (iii) We showcase our method using real-world driving data and demonstrate the utility of RA-CBFs and the learned responsibility allocations in safe closed-loop AV control and forensic analysis.


\section{Background} \label{sec:background}
The dynamics of the vehilces are abstracted as a nonlinear control-affine system of the form: 
\begin{align}
    \dot{\mb{x}} = \mb{f}(\mb{x}) + \mb{g}(\mb{x})\mb{u} \label{eq:ol_dyn}
\end{align}
where $\mb{x} \in \R^n$ and $\mb{u} \in \mathcal{U}\subset \R^m $ represent the states and inputs of the system and where $\mathcal{U}$ is assumed to be compact. We further assume that $\mb{f}: \R^{n} \to \R^{n}$ and $\mb{g} : \R^{n } \to \R^{ n \times m }$ are locally Lipschitz continuous. Given a state-feedback controller $\mb{k}: \R^n \to \R^m $, the closed-loop dynamics are: 
\begin{align}
    \dot{\mb{x}} = \mb{f}_\textrm{cl}(\mb{x}) = \mb{f}(\mb{x}) + \mb{g}(\mb{x}) \mb{k} (\mb{x}). \label{eq:cl_dyn}
\end{align}

\subsection{Control Barrier Functions}

In this section, we define safety as the forward-invariance of some \textit{safe set} $\mathcal{C}\subset \R^{n} $ and review Control Barrier Functions (CBFs) as a tool for synthesizing safe controllers.

\begin{definition}[Forward Invariance and Safety]
 A set $\mathcal{C}\subset \R^n$ is \textit{forward invariant} if for every $\mb{x}(0) \in \mathcal{C}$ the solution to \eqref{eq:cl_dyn} satisfies $\mb{x}(t) \in \mathcal{C}$ for all $t \geq 0 $. We call the system \eqref{eq:cl_dyn} \textit{safe} with respect to $\mathcal{C}$ if $\mathcal{C}$ is forward invariant. 
\end{definition}

Let the set $\mathcal{C}$ be the 0-superlevel set of some continuously differentiable function $h:\R^{n} \to \R$ with $0$ a regular value\footnote{We say that 0 is a regular value of $h$ if $h(\mb{x}) = 0 \implies \frac{\partial h }{\partial \mb{x} } \neq 0 $.}: 
\begin{align}
    \mathcal{C} &\triangleq \{ \mb{x} \in \R^n ~|~ h(\mb{x}) \geq 0 \} 
\end{align}
\noindent 
We deem $h$ a CBF if it also satisfies the following condition:

\begin{definition}[Control Barrier Function \cite{ames_control_2017}]
    Let $\mathcal{C}\subset \R^n$ be the 0-superlevel set of some continuously differentiable function $h: \R^{n} \to \R$ with 0 a regular value. The function $h$ is a \textit{Control Barrier Function (CBF)} for \eqref{eq:ol_dyn} if there exists an extended class $\mathcal{K}_\infty$ function\footnote{A continuous function $\alpha: \R_{\geq 0} \to \R_{\geq 0 } $ is a class $\mathcal{K}_\infty $ function if $\alpha(0) = 0$, $\alpha $ is strictly monotonically increasing, and $\lim_{c \to \infty} \alpha(c) = \infty$. A continuous function $\alpha: \R \to \R $ is an \textit{extended } $\mathcal{K}_\infty$ function if $\alpha(0) = 0$, $\alpha$ is strictly monotonically increasing, $\lim_{c \to \infty} \alpha(c) = \infty$, and $\lim_{c \to -\infty} \alpha(c) = - \infty$.  } $\alpha $ such that for all $\mb{x} \in \mathcal{C}$:  
    \begin{align}
        \sup_{\mb{u}\in \mathcal{U}} \quad   \overbrace{\underbrace{L_\mb{f}h(\mb{x})}_{\frac{\partial h }{\partial x}\mb{f}(\mb{x})  } + \underbrace{L_\mb{g}h(\mb{x})}_{\frac{\partial h}{\partial{x}}\mb{g}(\mb{x}) } \mb{u}}^{\frac{dh}{dt}(\mb{x}, \mb{u}) }\geq - \alpha (h(\mb{x})).
        \label{eq:cbf_constraint} 
    \end{align}
\end{definition}
\noindent Here, Lie derivative notation is used to represent the partial derivatives $L_\mb{f}h(\mb{x}) \triangleq \frac{\partial h}{\partial \mathbf{x}} \mb{f}(\mb{x})$ and $L_\mb{g}h(\mb{x}) \triangleq \frac{\partial h}{\partial \mathbf{x}} \mb{g}(\mb{x}) $. 
Intuitively, the CBF constraint \eqref{eq:cbf_constraint} limits $\frac{dh}{dt}$, and prevents $h$ from decreasing along the trajectory when $h(\mb{x}) = 0$, thus rendering $\mathcal{C}$ forward invariant for \eqref{eq:cl_dyn}. This intuition was formalized by \cite{ames_control_2017} in the following theorem: 

\begin{theorem}[CBF Safety, \cite{ames_control_2017}]
Given a set $\mathcal{C} \subset \R^n$ defined as the 0-superlevel set of a continuously differentiable function $h: \R^{n} \to \R$ with 0 a regular value, if $h$ is a CBF, then any locally Lipschitz controller $\mb{k}: \R^n \to \R^{m}$ that satisfies \eqref{eq:cbf_constraint} for all $\mb{x} \in \mathcal{C}$, renders the system \eqref{eq:cl_dyn} safe w.r.t. $\mathcal{C}$. \label{thm:cbf}
\end{theorem}



\section{Responsibility-aware Decentralized Multi-agent Safety} \label{sec:decentralized_safety}

In this section, we extend the CBF framework 
to a decentralized multi-agent setting and introduce additional terms to account for asymmetrically shared responsibility.

\subsection{Decentralization Multi-agent CBF} \label{subsec:multiagent}
We extend the system dynamics \eqref{eq:ol_dyn} to multiple agents: 
\begin{align}
    \dot{\mathbf{x}}_i  = \mathbf{f}_i(\mathbf{x}_i) + \mathbf{g}_i (\mathbf{x}_i) \mathbf{u}_i
\end{align}
\noindent where $\mb{x}_i \in \R^{n_i} $, $\mathbf{u}_i\in\mathcal{U}_i \subset \R^{m_i} $, $\mb{f}_i : \R^{n_i} \to \R^{n_i}$, $\mb{g}_i : \R^{n_i} \to \R^{m_i}$ represent the state, input, drift, and actuation matrix of agent $i$. For the entire system of $N\in \mathbb{N}$ agents, let $\mb{x} = \lmat \mb{x}_1^\top & \cdots &  \mb{x}_N^\top \rmat^\top $ denote the concatenated state and the dynamics for $\mb{x}$ be denoted as in \eqref{eq:ol_dyn}: 
\begin{align}
    \dot{\mb{x}} = \underbrace{\lmat \mb{f}_1(\mb{x}_1) \\ \vdots \\ \mb{f}_N(\mb{x}_N) \rmat}_{\mb{f}(\mb{x})} + \underbrace{\lmat \mb{g}_1(\mb{x}_1) \\ \vdots \\ \mb{g}_N(\mb{x}_N) \rmat}_{\mb{g}(\mb{x})} \underbrace{\lmat \mathbf{u}_1 \\ \vdots \\ \mathbf{u}_N \rmat}_{\mb{u}}.  \label{eq:decentralized_oldyn}
\end{align}

If the multi-agent system is governed by a centrallized controller, the CBF inequality can be checked directly and used as a constraint in an optimization-based controller to obtain safe inputs \cite{ames_control_2014}. However, centralized control is often unrealizable for AVs due to communication and scalability issues as well as the presence of human actors. Thus, we focus on a decentralized variant of the CBF constraint \eqref{eq:cbf_constraint} and assume that each agent can measure the states of the other agents, but independently generates its own input  according to some controller unknown to the other agents.

One common method for retaining safety guarantees in the context of decentralized control, is to ensure robustness with respect to all possible controls of the other agents (including the worst-case inputs) as in \cite{usevitch2022adversarial}. In this case, constraint \eqref{eq:cbf_constraint} from the perspective of agent $i$ becomes: 
\begin{align}
    \sup_{\mb{u}_i \in \mathcal{U}_i} \inf_{\substack{\mb{u}_j \in \mathcal{U}_j, \\  j\neq i }} L_\mb{f}h(\mb{x}) + L_\mb{g}h(\mb{x})\mb{u}\geq -\alpha(h(\mb{x})). 
    \label{eq:worst_case_cbf}
\end{align}
\noindent This is a conservative constraint which ensures the safety of the system even when other agents act adversarially. For human-interactive systems where responsibility is shared, this controller is often unnecessarily conservative.


Despite their safety-guarantees, worst-case constraints like \eqref{eq:worst_case_cbf} are highly conservative and prevent fluent behaviors  \cite{cosner2022safety}. It is therefore desirable to find a less conservative safety constraint that is more cognizant of the social interactions between agents even when the controllers of the other agents are unknown. For this purpose, we consider a novel CBF framework that models \textit{social responsibility}.

\subsection{Responsible-Aware Control Barrier Functions}

In multi-agent systems of human actors, the responsibility for maintaining safety is typically shared amongst people. For example, humans exhibit social behavior in crowd navigation and driving where the burden of maintaining safety is distributed amongst everyone
\cite{schwarting2019social, helbing1995social}. 
Equipped with the notion that agents share the responsibility of maintaining safety, we move away from worst-case behavioral assumptions, and instead, \textit{learn} the responsibility allocation from data.
First, we define \textit{responsibility allocation functions}: 
\begin{definition}[Responsibility Allocation Function]
    A function $\gamma: \mathbb{N} \times \R^n \to \R $ is a \textit{responsibility allocation function} for $N \in \mathbb{N}$ if for all $\mb{x}\in \R^n$:
    \begin{align}
        \sum_{i \in \{ 1, \dots, N \}} \gamma(i, \mb{x}) \geq 0, \label{eq:raf_property}
    \end{align}
\end{definition}
\noindent For agent $i$ in a multi-agent system at state $\mb{x}$, $\gamma(i, \mb{x}) > 0$ indicates increased responsibility, $\gamma(i, \mb{x}) = 0 $ indicates evenly shared responsibility, and  $\gamma(i, \mb{x})<0$ indicates decreased responsibility. The sum of $\gamma(i, \mb{x}) $ is lower bounded by zero to ensure that the total allocated responsibility must be greater than or equal to that of even sharing. 

Using these responsibility allocation functions we can present our definition of Responsibility-Aware Control Barrier Functions (RA-CBFs) which consider responsibility allocation in their decentralized multi-agent safety constraint: 
\begin{definition}[Responsibility-Aware Control Barrier Function]\label{def:RA-CBF}
Let $\mathcal{C}\subset \R^n $ be the 0-superlevel set of some continuously differentiable function $h: \R^n \to \R$ with 0 a regular value. Additionally, let $\gamma :\mathbb{N} \times \R^n  \to \R $ be a responsibility allocation function for $N \in \mathbb{N}$. The function $h$ is a \textit{Responsibility-Aware CBF} for the system \eqref{eq:decentralized_oldyn} and responsibility allocation function $\gamma$ if there exists an extended class $\mathcal{K}_\infty$ function $\alpha$ such that for all $\mb{x}\in \mathcal{C}$ and all $i \in \{1, \dots, N \} $ : 

\vspace{-2mm}
{\small
\begin{align}
    \sup_{\mb{u}_i \in \mathcal{U}_i}   \underbrace{\overbrace{L_{\mb{g}_i}h(\mb{x}) \mb{u}_i +  \frac{1}{N}\Big(\alpha(h(\mb{x})) + L_\mb{f}h(\mb{x})\Big)}^{\mb{c}_i(\mb{x}, \mb{u}) \triangleq} -  \gamma(i, \mb{x})}_{\textrm{RA-CBF Constraint}(i, \mb{x}, \mb{u}, \gamma) \triangleq} \geq 0 . \label{eq:ra-cbf_constraint}
\end{align}
}
\end{definition}

\begin{remark}
For generality, RA-CBFs are presented for $N$ agents but in practice it is common to enforce CBF constraints between each pair of agents  where the number of constraints enforced on agent $i$'s input grows linearly with the number of agents \cite{glotfelter_nonsmooth_2020}. In this case there would be several pairwise RA-CBF constraints with $N=2$. In Section \ref{sec:application}, we take this approach in our application.
\end{remark}


\begin{theorem}[Responsibility-Aware Safety]\label{thm:responsible_safety}
Given a set $\mathcal{C}\subset \R^n$ defined as the 0-superlevel set of a continuously differentiable function $h:\R^n \to \R$ with 0 a regular value, if $h$ is an RA-CBF for \eqref{eq:decentralized_oldyn} and the responsibility allocation function $\gamma: \mathbb{N} \times \R^n \to \R$ for $N\in \mathbb{N}$, then any locally Lipschitz controller $\mb{k}: \R^n \to \R^m $ that satisfies \eqref{eq:ra-cbf_constraint} for all $\mb{x}\in \mathcal{C}$ and $i \in \{1, \dots, N\}$, renders system \eqref{eq:cl_dyn} safe with respect to $\mathcal{C}$. 
\end{theorem}
\begin{proof}
First let $
\mb{c}_i(\mb{x}, \mb{k}(\mb{x})) \triangleq L_{\mb{g}_i}h(\mb{x}) \mb{k}_i(\mb{x}) + \frac{1}{N}\Big(\alpha(h(\mb{x})) + L_\mb{f}h(\mb{x})\Big) $
for all $i \in \{1, \dots, N \} $. Since the $\mb{k}_i$ satisfies \eqref{eq:ra-cbf_constraint}, 

\vspace{-2mm}
{\small
\begin{align}
   0 & \geq -\mb{c}_i(\mb{x}, \mb{k}(\mb{x}))+\gamma(i, \mb{x}), \\
   & \geq -\sum_{i \in \{ 1, \dots, N\}}  \mb{c}_i(\mb{x}, \mb{k}(\mb{x})) + \sum_{i \in \{ 1, \dots, N\}}  \gamma(i, \mb{x}), \label{pf:sum}\\
   & \geq  -\sum_{i \in \{ 1, \dots, N \}} \mb{c}_i(\mb{x}, \mb{k}(\mb{x}))\label{pf:resp_ineq}
\end{align}
}
Inequality \eqref{pf:sum} follows from the decentralized constraint \eqref{eq:ra-cbf_constraint} for all $i$  and \eqref{pf:resp_ineq} holds since $\gamma$ is a responsibility allocation function for $N$. 
Since the final inequality \eqref{pf:resp_ineq} is equivalent to the centralized CBF constraint \eqref{eq:cbf_constraint}, Theorem \ref{thm:cbf} implies the safety of system \eqref{eq:cl_dyn} with respect to $\mathcal{C}$. 
\end{proof}

In summary, instead of considering the worst-case inputs from other agents, our RA-CBF approach uses $\gamma(i,\mb{x})$ to allow agent $i$'s required contribution to decentralized multi-agent safety to vary depending on the state $\mb{x}$ of all of the agents in the scene. 
Also, instead of explicitly considering the uncertainty in the other agents' actions, one perspective of the responsibility allocation function is that it models a bound on the projection \cite{taylor2020control} of the other agents' inputs onto the CBF time derivative. Thus we can learn the effect of the other agents' actions as a scalar adjustment to $\frac{dh}{dt}$ as opposed to predicting their exact trajectories.
 
Our responsibility model is similar to that of  \cite{lyu2022responsibility} which instead uses a multiplicative term and is limited to driftless systems (i.e., systems where $\mb{f}(\mb{x}) \equiv 0$). By using an additive term, our model is generally applicable to control affine systems and is capable of accounting for the effect of responsibility even when the unforced dynamic  are unsafe, i.e. $\alpha(h) + L_\mb{f}h(\mb{x}) \leq 0 $. 

\section{Learning Responsibility Allocation} \label{sec:learning}

In this section we formalize the problem of learning responsibility allocation functions $\gamma(i,\mb{x})$ from data and describe our method for learning $\gamma$ from expert demonstrations, given a known safe set $\mathcal{C}$ and associated CBF $h$. 

\subsection{Problem Setting}
We assume that agent $i$ strives to minimize some unknown function $Q_i: \R^{n_i\times m_i } \to \R$ and does so according to a constrained optimal control policy: 
\begin{align}
    \mb{k}_i(\mb{x}) = 
    \argmin_{u_i \in \mathcal{U}_i} & \quad Q_i(\mb{x}_i, \mb{u}_i)  \label{eq:optimal_controller}\\
    \textrm{s.t. } & \quad \textrm{RA-CBF Constraint}(i,\mb{x}, \mb{u}, \gamma)  \geq 0 \nonumber
\end{align}
where $Q_i(\mb{x}_i, \mb{u}_i)$ represents agent $i$'s cost for input $\mb{u}_i$ at state $\mb{x}_i$. 
Although the cost function is unknown, we assume that all agents obey the RA-CBF constraint for some $\gamma$ that we seek to learn, thus framing the problem of learning responsibility allocations as constraint learning.

\subsection{Learning Paradigm}
Let $\mathcal{D} = \{ \mb{u}^k, \mb{x}^k \}_{k=1}^{N_d}$ be a dataset of state-input pairs gathered from expert (human) demonstrations where $N_d$ represents the total number of data points collected. Since the cost function $Q_i$ can vary during data collection, it is possible for a state to have several associated expert inputs.

Our goal is find some responsibility allocation function $\gamma$ such that the RA-CBF constraint is satisfied for all state-input pairs in the expert demonstrations $\mathcal{D}$. This can be written as the constrained optimization problem: 
\begin{align}
    \gamma^* &=
    \argmin_{\gamma}  \quad \Vert \gamma \Vert \label{eq:resp_alloc_learning}\\
    & \textrm{s.t. } \quad  \textrm{RA-CBF Constraint}(i, \mb{x}, \mb{u}, \gamma) \geq 0 ,\: \forall i\in \{1,...,N\},\nonumber\\
    & \quad \sum_{i \in \{ 1, \dots, N \} } \gamma(i,\mb{x}) \geq 0, \qquad \textrm{for all } (\mb{x}, \mb{u}) \in \mathcal{D},\nonumber
\end{align}
\noindent where the constraints enforce satisfaction of the RA-CBF and ensure that $\gamma$ is a responsibility allocation function.

To find an approximate solution to this problem we take inspiration from \cite{robey2020learning,qin2021learning} and relax \eqref{eq:resp_alloc_learning} to the following unconstrained loss function: 

\vspace{-2mm}
{\small
\begin{align} 
    \mathcal{L}(\mathcal{D}, \gamma) = 
    &  \Vert \gamma \Vert + \lambda_1 \sum_{(\mb{x}, \mb{u}) \in \mathcal{D}} \sum_{i=1}^N\bigg[ -\mb{c}_i(\mb{x}, \mb{u}) + \gamma(i, \mb{x})    \bigg]_+ \nonumber  \\
    & +  \lambda_2 \sum_{(\mb{x}, \mb{u}) \in \mathcal{D}}\left[\sum_{i=1}^N -\gamma_i(i, \mb{x}) \right]_+  \label{eq:loss}
\end{align}
}
where $\lambda_1, \lambda_2, \in \R_{\geq 0}$ are hyperparameters which adjust the constraint relaxations and $ [ \; \cdot \; ]_+ \triangleq \max \{ \cdot, 0 \}  $. This loss function can then be used find approximate solutions to \eqref{eq:resp_alloc_learning}: 
\begin{align}
    \gamma^* \approx \argmin & \quad  \mathcal{L}(\mathcal{D}, \gamma) \label{:unreg_opt}
\end{align}

\begin{figure*}[t]
    \centering
    \includegraphics[width=\linewidth]{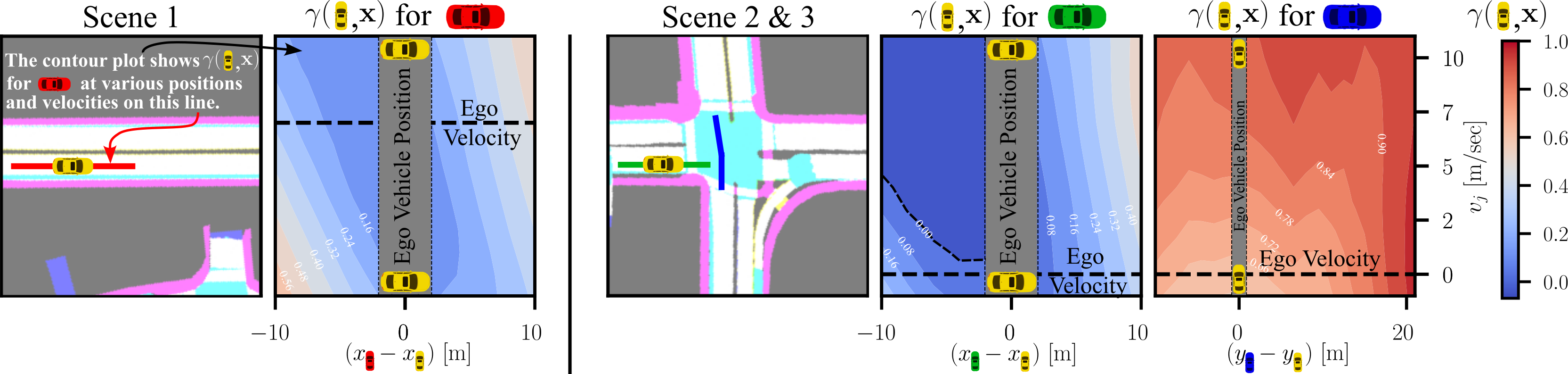}
    \caption{The learned responsibility allocation surface is visualized for a range of velocities and relative positions. \textbf{Scene 1: } the ego vehicle (yellow) is driving on a two lane road. In all cases, $\gamma(\textrm{ego}, \mb{x}) > 0$ indicating a degree of conservative driving. Generally, $\gamma(\textrm{ego},\mb{x})$ is larger when the other agent (red) is in front of the ego vehivlethan when behind, indicating increased responsibility when driving behind another vehicle. \textbf{Scene 2: } The ego vehicle is stopped at a four-way intersection with the other vehicle (green) ahead or behind it (and no blue agent). Again the ego vehicle (yellow) is more responsible when the green vehicle is in front of it than when it is behind it. \textbf{Scene 3: } The ego vehicle (yellow) is stopped at a four-way intersection with the other vehicle (blue) crossing from top to bottom (and no green agent). $\gamma(\textrm{ego}, \mb{x}) $ is large for all positions and velocities of the blue vehicle showing that the ego agent takes is more responsible in this situation.} 
    \vspace{-0.5cm}
    \label{fig:resp_vis}
\end{figure*}

\subsection{Responsibility Regularization}
However, the loss function used in the unconstrained optimization \eqref{:unreg_opt} is insufficient since, 
as in Inverse Reinforcement Learning, the problem of learning the constraint in \eqref{eq:optimal_controller} is poorly defined since the optimal input generated by \eqref{eq:optimal_controller} is a function of both the unknown cost function $Q_i$ and unknown responsibility allocation function $\gamma$. 
Intuitively, this is because we cannot answer the question ``did the agent act that way because it wanted to (cost minimization) or because it had to (safety constraint satisfaction)?". To better define the constraint learning problem we take an approach similar to \cite{scobee2019maximum} and regularize $\gamma$ by maximizing the likelihood that it was used in \eqref{eq:optimal_controller} to generate the expert demonstrations $\mathcal{D}$. 

Following the maximum entropy model presented in \cite{ziebart2008maximum}  with the variant for continuous-time nonlinear systems presented in \cite{aghasadeghi2011maximum} we wish to solve the optimization problem: 
\begin{align}
    \gamma^*_\textrm{reg} = \argmax \sum_{(\mb{x}, \mb{u}) \in \mathcal{D}}\mathcal{P}(\mb{u} ~|~ \mb{x}, \gamma).
\end{align}
We approximate the probability of a given $\mb{u}$, by choosing $\textrm{disc}(\mathcal{U})$ to be a finite discretization of $\mathcal{U}$ such as $\textrm{disc}(\mathcal{U}) = \{\mb{u} \in \mathcal{U} \;|\;  \delta\lfloor  \mb{u}/\delta \rceil \}  $ for some $\delta>0$ where $\lfloor \cdot \rceil $ rounds each component to the nearest integer. 
Mimicking the forms presented in \cite{aghasadeghi2011maximum, scobee2019maximum}, the approximate probability of an input $\mb{u} \in  \mathcal{U}$  given the system state $\mb{x}$ and responsibility allocation $\gamma$  is: 

\vspace{-3mm}
{\small
\begin{align}
    \mathcal{P}(\mb{u} ~ | ~ \mb{x}, \gamma) = \frac{e^{R(\mb{x},\mb{u})}}{Z_\gamma} \mathds{1}^\gamma(\mb{x},\mb{u}), \: Z_\gamma = \sum_{\bs{\upsilon} \in \textrm{disc}(\mathcal{U})} e^{R(\mb{x},\bs{\upsilon})} \mathds{1}^\gamma(\mb{x}, \bs{\upsilon}) \label{eq:maxlike}
\end{align}
}
\noindent where $Z$ is the partition function, $R: \R^n \times \R^m \to \R $ is the reward function, 
and $\mathds{1}^\gamma(\mb{x}, \mb{u}) \mapsto \{0, 1 \}$ indicates satisfaction of the RA-CBF constraints given $\mb{x},\; \mb{u},$ and $\gamma$.

To maximize the likelihood of the demonstration we minimize the number of feasible inputs 
while retaining the feasibility of the expert demonstrations.
We note 
the total number of feasible inputs in $\textrm{disc}(\mathcal{U})$ decreases as  $\gamma(i, \mb{x})$ increases regardless of $R$, 
so we can maximize $\mathcal{P}(\mb{u}|\mb{x}, \gamma)$ without knowledge of the agents' reward functions by maximizing  $\gamma$ while maintaining feasibility of the expert demonstrations. This can be expressed as the optimization problem:
\begin{align}
    \gamma_\textrm{reg}^* \approx \argmax_{\gamma} \quad & \sum_{(\mb{x}, \mb{u})\in \mathcal{D}} \sum_{i=1}^N \gamma(i, \mb{x})\\
    \textrm{s.t. } \quad &\textrm{RA-CBF Constraint}(i, \mb{x}, \mb{u}, \gamma) \geq 0,  \nonumber\\
    & \textrm{for all } i \in \{1, \dots, N\} \textrm{ and } (\mb{x}, \mb{u}) \in \mathcal{D} \nonumber 
 \end{align}

Since constraint feasibility on $\mathcal{D}$ is already accounted for in \eqref{eq:loss}, this regularization can be added to the loss as:
\begin{align}
    \mathcal{L}_\textrm{reg}(\mathcal{D}, \gamma) \triangleq \mathcal{L}(\mathcal{D}, \gamma) + \lambda_3 \sum_{(\mb{x}, \mb{u}) \in \mathcal{D}} \sum_{i=1}^N-\gamma(i, \mb{x}) 
\end{align}
with hyperparamter $\lambda_3 \in \R_{\geq 0}$  
which can be used in the final regularized optimization problem to estimate $\gamma$:  
\begin{align}
    \gamma^* = \argmin \; \mathcal{L}_\textrm{reg}(\mathcal{D}, \gamma) 
\end{align}


\section{Application to Autonomous Driving} \label{sec:application}

In this section we apply our RA-CBF and responsibility allocation learning method to urban driving using the Boston Seaport data provided by the nuScenes dataset \cite{nuscenes}.

We assume that all agents in the scene are vehicles (i.e. there are no pedestrians) and we model each agent as:  
\begin{align}
    \underbrace{\lmat \dot x_i \\ \dot y_i \\ \dot v_i \\ \dot \theta \rmat }_{\dot{\mb{x}}_i} = \underbrace{\lmat v_i \cos(\theta_i) \\ v_i \sin(\theta_i) \\ 0 \\ 0  \rmat}_{\mb{f}_i(\mb{x}_i)} + \underbrace{\lmat 0 & 0 \\ 0 & 0 \\ 1 & 0 \\ 0 & 1  \rmat}_{\mb{g}_i(\mb{x}_i)} \underbrace{\lmat a_i \\ \omega_i \rmat }_{\mb{u}_i},  \label{eq:veh_cl_dyn}
\end{align}
\noindent where $(x_i,y_i) \in \R^2$, $v_i ,\theta_i ,a_i,\omega_i \in \R$ represent the position, velocity, yaw, acceleration, and yaw rate\footnote{Bezier curves are fit to position and yaw data and then differentiated to obtain velocity, acceleration, and yaw rate. The code repository for learning the responsibility allocation function $\gamma$ can be found \href{https://github.com/rkcosner/learning_responsibility_allocation}{\textbf{here}}.} of vehicle $i$.


\subsection{Choosing a Safety Metric for Autonomous Driving}
To define the safety function $h: \R^n \to \R $, we begin by assuming all vehicles must maintain a minimum inter-vehicle distance $\underline{d}>0$. With this in mind, let $d_{\textrm{min}}: \R^{4} \times \R^{4} \to \R $ be the minimum distance between two agents. 
We can then define the pairwise safe set between agents $i$ and $j$ to be:
\begin{align}
    \mathcal{C}_{ij} = \big\{ \mb{x} \in \R^n ~|~ \underbrace{d_{\textrm{min}}(\mb{x}_i, \mb{x}_j)-\underline{d}}_{h_{ij}(\mb{x})\triangleq} \geq 0 \big\}.
\end{align}
However this function is of relative degree 2 w.r.t. $a_i$ (i.e., $\frac{dh}{dt}$ is not directly affected by $a_i$) and describes safety by only considering the instantaneous current position. 


In order to incorporate a temporal aspect and ensure that the time derivative of $h_{ij}$ is affected by both vehicles' acceleration and angle rate we take inspiration from \cite{gurriet2020scalable} and forward project the current state using a \textit{backup controller} $\mb{k}_B:\R^{4} \to \R^{2}$ over a time interval $[0,T]$ for $T \in \R_{>0}$. By assumption, for any $\mb{x}_i(t) \in \R^{4} $ there exists a unique solution $\bs{\phi}: [0,T] \to \R^4 $ satisfying: 
\begin{align}
    \frac{d}{dt}\bs{\phi}(\tau) & = \mb{f}_i(\bs{\phi}(\tau)) + \mb{g}_i(\bs{\phi}(\tau))\mb{k}_B(\bs{\phi}(\tau)), \label{eq:flow_diffeq}\\
    \bs{\phi}(0) & = \mb{x}_i(t) \label{eq:flow_iv}
\end{align}
The solution $\bs{\phi}$ starting at $\mb{x}_i(t)$ is the flow under $\mb{k}_B$, and is  denoted as $\bs{\varphi}_\tau(\mb{x}_i) \triangleq \bs{\phi}(\tau)$.
Similar ideas of forward projection are also seen in Velocity Obstacles \cite{wilkie2009generalized}, Safety Force Fields \cite{nister2019safety}, and Responsibility-Sensitive Safety \cite{shalev2017formal}. 

Using the flow $\bs{\varphi}$ and the distance function $d_{ij}$ we can find the minimum distance that would be achieved during the interval $[t, T+t]$ if $\mb{k}_B$ were the controller for both vehicles:  
\begin{align}
    h^{\bs{\varphi}}_{ij}(\mb{x}) = \min_{\tau \in [0, T] } d_\textrm{min}(\bs{\varphi}_{\tau}(\mb{x}_i), \bs{\varphi}_\tau(\mb{x}_j))   - \underline{d}, \label{eq:flow_barrier}
\end{align}
which has the associated safe set $\mathcal{C}_{ij}^\varphi \subseteq \mathcal{C}_{ij} \subset \R^n $ 
\begin{align}
    \mathcal{C}_{ij}^{\bs{\varphi}} = \left\{ \mb{x} \in \R^n ~|~ h^{\bs{\varphi}}_{ij}(\mb{x}) \geq 0 \right\}.
\end{align}
To compute $h^{\bs{\varphi}}_{ij}$ the interval $[0,T]$ was discretized at 100 Hz as in \cite{gurriet2020scalable, cosner_measurement-robust_2021} and soft minimum functions were used to ensure differentiability.  It is shown in \cite{chen2021backup} that the CBF $h^{\bs{\varphi}}_{ij}$ constructed from the backup controller is guaranteed to be of relative degree 1 under mild assumptions.

For the backup controller we choose $\mb{k}_B(\mb{x}_i) = \mb{0} $ which approximates idling. Unlike other methods \cite{nister_visual_2004, shalev2017formal} which assume maximum braking for their predictors, we choose an idling controller since it better approximates nominal driving behavior and does not introduce worst-case assumptions.

Given these pairwise safe sets $\mathcal{C}_{ij}^\varphi$, we can define the a global safe set $\mathcal{C}^\varphi \subseteq \mathcal{C}_{ij}^\varphi \subseteq \mathcal{C}_{ij}$ for all $i \neq j$  as: 
\begin{align}
    \mathcal{C}_{ij}^\varphi  = \bigcap_{i\neq j }\mathcal{C}_{ij}^\varphi  \quad \textrm{ with }\quad   h(\mb{x}) = \min_{i\neq j} h_{ij}^{\bs{\varphi}}(\mb{x}). \label{eq:final_cbf}
\end{align}
 \noindent The intersection of safe sets has been studied in \cite{glotfelter_hybrid_2019} and safety of such sets can be achieved by enforcing the safety constraint for all $i\neq j$. 
 
 \begin{figure*}[t]
    \centering
    \includegraphics[width=0.99\linewidth]{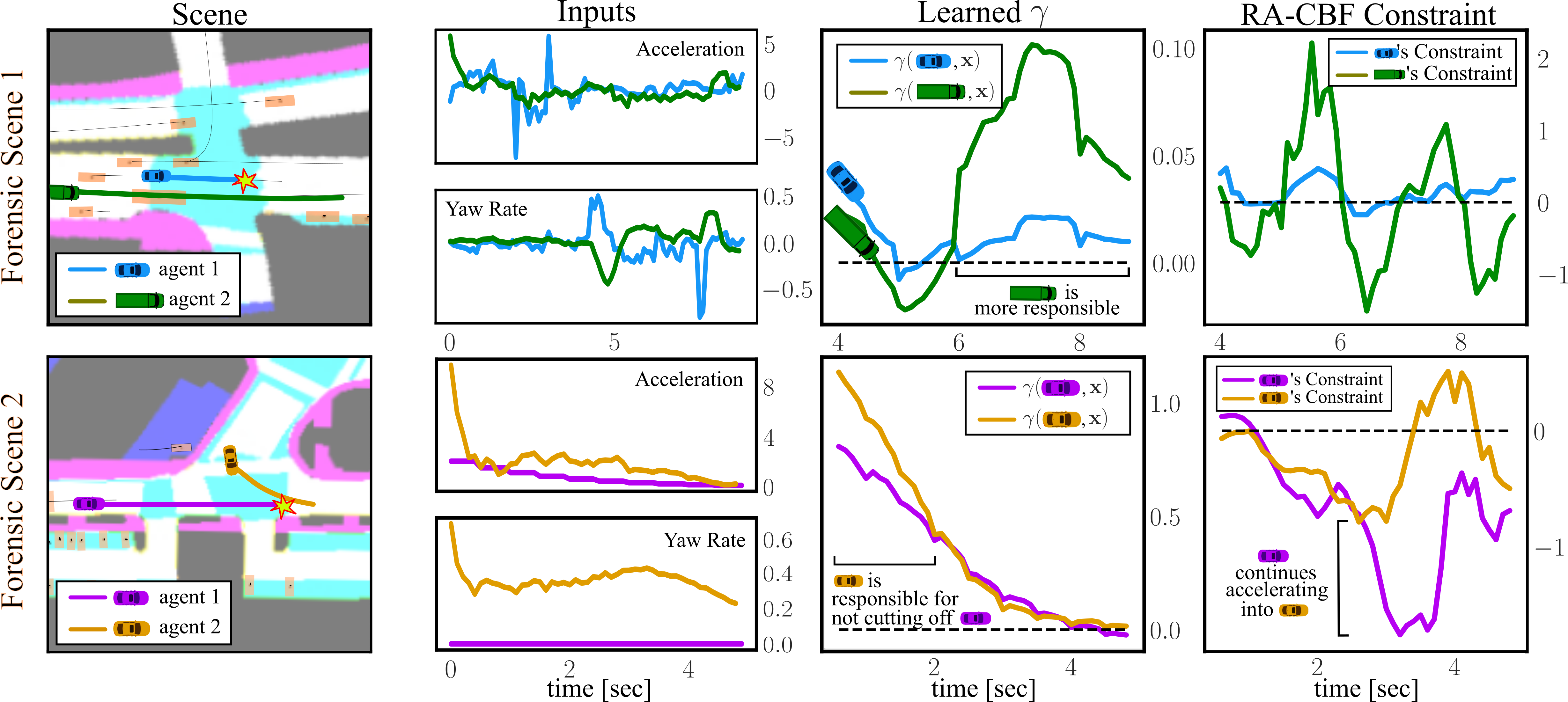}
    \caption{Crash scenes for forensic analysis. In both rows the plots from left to right display: the scene and trajectories, the inputs for agents 1 and 2, the learned responsibility allocations for each agent, and the values of their RA-CBF constraints \eqref{eq:ra-cbf_constraint}.}
    \label{fig:forensic}
    \vspace{-0.4cm}
\end{figure*}

\begin{figure}[t]
\vspace{-1em}
\begin{minipage}[t]{\linewidth}
\begin{table}[H]
\centering
Closed-Loop Simulation Results
\scriptsize
    \begin{tabular}{|c|c|c|c|}    
        \hline  & Worst-Case & Even-Split & Our Method  \\
         \hline Validation Constraint Violation & 43.99\% & 8.13\% &  9.51\%  \\
         \hline Closed-Loop Safety Violation &  0.833\% & 2.50\%  & 0.833\%  \\
         \hline Time Spent Off Road &  1.48\% & 0.59\% & 0.54\%  \\ 
         \hline Distance Covered Metric  & 290.21 &  309.21 & 307.84  \\
         \hline
    \end{tabular}
    \caption{Results for the closed-loop experiments.
    }
    \label{tbl:experiment_results}
\end{table}
\end{minipage}
\vspace{-0.5cm}
\end{figure}

\subsection{Learning Model}

The inputs of the responsibility allocation function are a semantic image as see in Fig. \ref{fig:resp_vis} and the relative vehicle states of agents $i$ and $j$. The image is processed by ResNet-18 \cite{he2015resnet} and the 256 dimensional output is concatenated to the vehicle states and processed by a multi-layer perceptron (MLP) with 2 hidden layers of size 128 and a single dimensional output.


The hyperparameters chosen were $\lambda_1 = 1$, $\lambda_2 = 10$, $\lambda_3 = 0.01$, $\alpha = 0.5$, $T = 1 $, $\underline{d}=0.4$,  $\ell_1=0.1$, $\ell_2 =0.01$ and $\theta_\textrm{max} = 100^\circ$ where $\ell_1$ and $\ell_2$ were the negative slopes of the MLP's leaky ReLU activation functions and $\theta_\textrm{max}$ is used to filter the dataset such that only interactions between vehicles whose headings are within $\pm \theta_\textrm{max}$ are considered. The parameter $\theta_\textrm{max}$ is necessary since our data does not include lane direction annotation. We note that this does limit the applicability of this network and plan to include lane direction information in future work.

The network was trained on the NuScenes Boston Seaport dataset. Example responsibilities generated by our learned model can be found in Fig. \ref{fig:resp_vis}. These figures show that our model conforms to the general intuition that the vehicle behind is more responsible than the vehicle in front for avoiding collisions between them, and the vehicle stopped at an intersection is responsible for not interfering with a vehicle already crossing the intersection.

\subsection{Closed-Loop Testing}
We use our RA-CBF framework with a learned responsibility allocation function as a safety-filter in closed-loop control
and simulate human-like driving using the Bi-Level Imitation for Traffic Simulation (BITS) model \cite{xu2022bits}. The ego agent follows \eqref{eq:optimal_controller} where $Q_i(\mb{x}_i, \mb{u}_i) =  \Vert \mb{k}_{\textrm{bits}}^{+}(\mb{x}_i) - \mb{u}_i \Vert^2$ and $\mb{k}_{\textrm{bits}}^{+}$ is the BITS controller with an additional 1 $\frac{\textrm{m}}{\textrm{sec}^2}$ acceleration added to generate irresponsible desired behavior that must be filtered to ensure safety. The RA-CBF constraint is applied for each pairwise vehicle $j$ and slack variables are used to ensure feasibility. We compare our method to the same controller with two other baseline constraints: (i) \textit{``Worst-Case''} constraint \eqref{eq:worst_case_cbf}, and (ii) \textit{``Even-Sharing''} constraint which is the RA-CBF constraint with $\gamma(i, \mb{x}) \equiv 0$. 

The closed-loop system was run in 120 scenarios sampled from NuScenes for 10 seconds at 10 Hz. Table \eqref{tbl:experiment_results} contains metrics comparing the controllers. The Worst-Case controller has the fewest safety violations (as expected), but worse compatibility with the expert demonstrations as indicated by the large constraint violation, smallest distance covered, and significant amount of time off of drivable surfaces. The Even-Sharing controller has fewer constraint violations on the validation data and the most distance covered, but allows for more collisions. Our method has a slightly higher number of constraint violations, but achieves a better safety-performance trade-off.


\subsection{Forensic Analysis}

In addition to closed loop control, the values of $\gamma$ and the RA-CBF constraint provide useful insight when performing forensic analysis on unsafe driving behaviors. To demonstrate this, we analyze the two collision scenarios shown in Fig. \ref{fig:forensic}.  

\noindent\textbf{Forensic Scenario 1:} In the first scene, agent 1 is crossing an intersection as agent 2 approaches from behind. For more than 2 seconds preceding the crash, $\gamma(2, \mb{x}) > \gamma(1, \mb{x})$ indicating that agent 2 should have taken a greater share of the responsibility. However, agent 2 has several large violations of the RA-CBF constraint while agent 1 generally satisfies it. From this we interpret that agent 2 is responsible for the crash which aligns with our intuition.

\noindent \textbf{Forensic Scenario 2: } In the second scene, agent 1 is driving along a road that agent 2 is turning on to. At first, agent 2 is more responsible for not entering the lane in front of agent 1 $(i.e., \gamma(2, \mb{x}) > \gamma(1, \mb{x}) $ for $t < 2$). However, the responsibility allocation values then become similar as vehicle 2 enters the lane and accelerates. Throughout the scene, both agents violate safety, but agent 1 has much larger constraint violations in the moments preceding the crash due to its continued acceleration. This can be interpreted to indicate that although it was irresponsible of agent 2 to merge when it did, agent 1 is ultimately culpable.

We recognize that this analysis is subjective and one of many possible interpretations, but we believe that our method can provide useful insight when performing forensic analysis.




\ifthenelse{\boolean{separate}}{\clearpage}{} 

\section{Conclusion} \label{sec:conclusion}

We have presented Responsibility-Aware Control Barrier Functions (RA-CBFs) as a framework to learn and synthesize safe and responsible driving behaviors. RA-CBFs are designed to capture the asymmetric sharing of responsibility between multiple (human) agents and we present a method to learn context-dependent responsibility allocations from data. 
We then demonstrated the efficacy and utility of our approach using real-world driving data.
This work enables various exciting future directions which include incorporating explicit traffic rules into our responsibility-learning paradigm, comparing how responsibility allocations vary across geographical regions, and exploring other application domains such as crowd navigation. 

\clearpage







\ifthenelse{\boolean{separate}}{\clearpage}{} 





\ifthenelse{\boolean{separate}}{\clearpage}{} 











\balance
\bibliographystyle{IEEEtran}
\bibliography{cosner_main}


\end{document}

%% file: preamble.tex
\usepackage{hyperref}
\usepackage{graphicx}		
\usepackage{wrapfig}
\usepackage[format=plain,font=footnotesize,labelfont=bf,labelsep=period]{caption}
\usepackage{sidecap} 
\usepackage[export]{adjustbox}
\usepackage{subcaption}

\usepackage{amsmath}
\usepackage{amssymb}
\usepackage{bbm}
\usepackage{dsfont}
\usepackage{mathtools}
\usepackage{algorithm}
\usepackage{algpseudocode}

\usepackage{pdfsync}
\usepackage[normalem]{ulem}
\usepackage{paralist}	
\usepackage[space]{grffile} 
\usepackage[noabbrev, capitalise]{cleveref}
\usepackage{color}
\usepackage{dirtytalk}

\usepackage{amsthm}
\usepackage{balance}

\newtheorem{theorem}{Theorem}

\theoremstyle{definition}
\newtheorem{definition}{Definition}
\newtheorem{remark}{Remark}
\theoremstyle{definition}

\theoremstyle{definition}

\newcommand{\R}{\mathbb{R}}

\definecolor{darkblue}{RGB}{0,0,102}
\definecolor{lightblue}{RGB}{77,77,148}

\definecolor{gold}{RGB}{234, 170, 0}
\definecolor{metallic_gold}{RGB}{139, 111, 78}

\newcommand{\mb}[1]{\mathbf{ #1 }}
\newcommand{\bs}[1]{\boldsymbol{ #1 }}

\DeclareMathOperator*{\argmin}{argmin}
\DeclareMathOperator*{\argmax}{argmax}

\newcommand{\lmat}{\begin{bmatrix}}
\newcommand{\rmat}{\end{bmatrix}}
\newcommand{\lcase}{\begin{cases}}
\newcommand{\rcase}{\end{cases}}

\usepackage{ifthen}
\newboolean{separate}
\setboolean{separate}{false}

\usepackage{dsfont}